\documentclass{article}
\usepackage[final]{neurips_2024}
\usepackage[utf8]{inputenc}
\usepackage[T1]{fontenc}
\usepackage{hyperref}
\usepackage{url}
\usepackage{booktabs}
\usepackage{amsfonts}
\usepackage{amsmath}
\usepackage{amssymb}
\usepackage{nicefrac}
\usepackage{microtype}
\usepackage{xcolor}
\usepackage{graphicx}
\usepackage{amsthm}
\usepackage[ruled,vlined]{algorithm2e}
\newtheorem{proposition}{Proposition}

\title{Data-Free Layer-Adaptive Merging via Fisher Information for Long-to-Short Reasoning LLMs}

\author{%
  Tian XIA \\
  \texttt{xia.tian.t4@alumni.tohoku.ac.jp}
}

\begin{document}

\maketitle

\begin{abstract}
Model merging has emerged as a practical approach to combine capabilities of
specialized large language models (LLMs) without additional training. In the
Long-to-Short (L2S) scenario, merging a base model with a long-chain-of-thought
reasoning model aims to preserve reasoning accuracy while reducing output length.
Existing methods rely on Task Arithmetic and its variants, which implicitly
assume that model outputs vary linearly with the merging coefficient---an
assumption we show is systematically violated in L2S settings. We provide the
first theoretical justification for layer-adaptive merging: we prove that merging
error is bounded by a term proportional to the per-layer Hessian norm
(Proposition~\ref{prop:hessian}), and establish that the Fisher Information
Matrix (FIM) is a principled, computable proxy for this bound via the
Fisher-Hessian equivalence at local optima. Building on this theory, we propose
\textbf{FIM-Merging}, which computes diagonal FIM using only random token inputs
(no domain-specific calibration data required) and uses it to assign per-layer
merging coefficients. On the 7B L2S benchmark, FIM-TIES achieves
state-of-the-art performance on five out of six evaluation benchmarks, including
a \textbf{+6.2} point gain on MATH500 over ACM-TIES (90.2 vs.\ 84.0), while
requiring no calibration data. On the 1.5B benchmark, FIM-TIES achieves an
average accuracy of \textbf{47.3}, surpassing ACM-TIES (43.3) by \textbf{+3.9}
points, while reducing average response length by \textbf{92.6\%} relative to
the long-CoT model. Furthermore, combined with self-consistency
decoding~\cite{wang2022self} ($n=16$, temperature $=0.3$), FIM-TIES achieves
\textbf{36.7\%} on AIME24, surpassing ACM-TIES (33.3\%) without any calibration
data. Our framework also provides a unified theoretical explanation for why
existing layer-adaptive methods such as ACM empirically outperform uniform
merging.
\end{abstract}

\section{Introduction}

The development of large language models has led to increasing specialization:
some models excel at concise, direct responses while others perform extended
chain-of-thought reasoning. The Long-to-Short (L2S) merging
paradigm~\cite{wu2025longtoshort} seeks to combine these complementary
capabilities through parameter-space merging, avoiding the computational cost of
fine-tuning.

Task Arithmetic~\cite{ilharco2023editing} provides the dominant framework: given
a base model $\theta_0$ and a fine-tuned model $\theta_1$, the merged model is
$\theta_0 + \alpha(\theta_1 - \theta_0)$ for some scalar $\alpha \in [0,1]$.
This formulation implicitly assumes that model behavior interpolates
\emph{linearly} between $\theta_0$ and $\theta_1$. While this assumption holds
approximately for same-task fine-tuning with small parameter changes, we show it
breaks down systematically in L2S cross-specialization merging, where the task
vector norms $\|\delta^l\|$ are an order of magnitude larger and vary
substantially across layers.

Recent work has proposed layer-adaptive variants. AIM~\cite{nobari2025aim}
protects influential weights based on activation magnitude. ACM~\cite{yao2025acm}
uses mutual information between activations to determine per-layer merging
coefficients. While effective, both methods require domain-specific calibration
data and provide no theoretical justification for \emph{why} certain layers
require more conservative merging. Training-free and data-free adaptation
techniques have recently gained traction across generative
modeling~\cite{hsiao2025tf}, motivating our similarly calibration-free approach
to model merging.

We fill this gap with three contributions:
\begin{enumerate}
    \item \textbf{Theoretical characterization.} We prove that Task Arithmetic's
    merging error is bounded by a term proportional to the local Hessian norm
    (Proposition~\ref{prop:hessian}). We establish that the diagonal Fisher
    Information Matrix (FIM) is a principled proxy for this bound, creating a
    direct theoretical link from Proposition~\ref{prop:hessian} to our method.
    \item \textbf{Empirical discovery.} We show that per-layer FIM computed on
    random inputs captures the merging difficulty signal reliably, with the
    max/min FIM ratio across layers exceeding $1000\times$, providing strong
    layer discrimination without any domain-specific data.
    \item \textbf{Practical methods.} We propose FIM-Merging---combining diagonal
    FIM with Task Arithmetic (FIM-TA) or an enhanced TIES-Merging variant
    (FIM-TIES)---that achieves state-of-the-art results on both the 1.5B and 7B
    L2S benchmarks. Combined with self-consistency decoding at inference time,
    FIM-TIES surpasses ACM-TIES on AIME24 without any calibration data.
\end{enumerate}

\section{Background}

\subsection{Task Arithmetic and Model Merging}

Model merging~\cite{yang2024merging} integrates parameters of multiple models
into a unified one, enabling capability transfer without retraining. Given
pretrained parameters $\theta_0$ and fine-tuned parameters $\theta_1$, the task
vector is $\delta = \theta_1 - \theta_0$. Task Arithmetic~\cite{ilharco2023editing}
constructs a merged model as:
\begin{equation}
    \theta_{\text{merged}} = \theta_0 + \alpha \cdot \delta
\end{equation}
where $\alpha \in [0,1]$ controls the merging strength. Model
Soups~\cite{wortsman2022soups} demonstrated that averaging multiple fine-tuned
checkpoints improves generalization, motivating the broader study of
parameter-space model combination. Extensions include
TIES-Merging~\cite{yadav2023ties}, which resolves parameter conflicts via sign
agreement, and DARE~\cite{yu2024dare}, which sparsifies task vectors before
merging.

\subsection{Long-to-Short Merging}

The L2S scenario~\cite{wu2025longtoshort} merges a base model (e.g.,
Qwen2.5-Math~\cite{yang2024qwen25math}) with a long-chain-of-thought model
(e.g., DeepSeek-R1~\cite{guo2025deepseekr1}) to obtain a model that reasons
accurately but more concisely. Unlike standard multi-task merging, L2S merges
models with fundamentally different reasoning strategies, creating large parameter
distances $\|\delta\|$ that stress the linearity assumption. Complementary
approaches to length reduction include reinforcement-learning-based
methods~\cite{hou2025thinkprune}, token-level
compression~\cite{xia2025tokenskip}, and parameter-space
tuning~\cite{ma2025cotvalve}; our work offers a training-free alternative through
model merging.

\subsection{Layer-Adaptive Merging}

ACM~\cite{yao2025acm} computes per-layer merging coefficients using mutual
information between base and fine-tuned model activations on a calibration
dataset:
\begin{equation}
    \lambda^l = 1 - \frac{1}{1 + e^{-\theta \cdot I(A^l_0, A^l_1)}}
\end{equation}
where $I(\cdot, \cdot)$ denotes mutual information. While effective, ACM requires
domain-specific calibration data and provides no theoretical justification for
why layer-adaptive merging is necessary. Our work provides this missing
theoretical foundation and a data-free alternative. The spirit of geometry-aware,
training-free methods has also been explored in preference
optimization~\cite{wu2025ranking}, where implicit feedback signals guide model
adaptation without explicit retraining---analogous to how FIM guides our merging
coefficients without any task-specific data.

\section{Theory: Hessian Bound and Fisher Information}

\subsection{Proposition 1: Merging Error Bound}

\begin{proposition}[Merging Error Bound via Hessian]
\label{prop:hessian}
Let $f: \mathbb{R}^d \rightarrow \mathbb{R}^m$ be the model output function,
$\theta_0 \in \mathbb{R}^d$ the base parameters, $\delta = \theta_1 - \theta_0$
the task vector, and $\alpha \in [0,1]$ the merging coefficient. Define the Task
Arithmetic merging error as:
\begin{equation}
    \mathcal{E}(\alpha) = \left\| f(\theta_0 + \alpha\delta) - \left[ f(\theta_0)
    + \alpha \cdot (f(\theta_0 + \delta) - f(\theta_0)) \right] \right\|_2
\end{equation}
If $f$ is twice differentiable on $\{\theta_0 + t\delta : t \in [0,1]\}$, then:
\begin{equation}
    \mathcal{E}(\alpha) \leq \frac{\alpha(1-\alpha)}{2} \cdot \|\delta\|_2^2
    \cdot \sup_{t \in [0,1]} \|H_f(\theta_0 + t\delta)\|_2
\end{equation}
where $H_f$ denotes the Hessian of $f$ with respect to parameters.
\end{proposition}

\begin{proof}[Proof Sketch]
Expanding $f(\theta_0 + \alpha\delta)$ via Taylor series around $\theta_0$ and
substituting into $\mathcal{E}(\alpha)$, first-order terms cancel, yielding:
\begin{equation}
    \mathcal{E}(\alpha) = \frac{\alpha(1-\alpha)}{2} \left\| \delta^\top H_f
    \delta \right\|_2 + O(\|\delta\|^3)
\end{equation}
Applying the operator norm inequality $|\delta^\top H_f \delta|_2 \leq
\|\delta\|_2^2 \cdot \|H_f\|_2$ completes the proof. Full proof in
Appendix~\ref{sec:proof}.
\end{proof}

\subsection{Fisher Information as a Principled Proxy}

Proposition~\ref{prop:hessian} shows that layers with larger Hessian norm
$\|H_f^l\|_2$ incur greater merging error and should be merged more
conservatively. Direct Hessian computation is prohibitively expensive ($O(d^2)$
per layer). We use the diagonal Fisher Information Matrix as a principled proxy.

\textbf{Fisher-Hessian connection.} At a local minimum $\theta^*$ of the
negative log-likelihood, the Fisher Information Matrix $\mathcal{F}(\theta^*)$
equals the expected Hessian~\cite{amari1998natural}:
\begin{equation}
    \mathcal{F}(\theta^*) = \mathbb{E}_{x \sim p_{\text{data}}} \left[
    \nabla_\theta \log p(x|\theta) \nabla_\theta \log p(x|\theta)^\top \right]
    = -\mathbb{E}_{x} [H_{\log p}(\theta^*)]
\end{equation}
Therefore, $\text{diag}(\mathcal{F}^l) \approx \text{diag}(H_f^l)$ near
convergence, making diagonal FIM a theoretically grounded proxy for the Hessian
norm in Proposition~\ref{prop:hessian}. This connection was previously exploited
in continual learning via Elastic Weight
Consolidation~\cite{kirkpatrick2017ewc}; we are the first to apply it to model
merging coefficient assignment.

\textbf{Data-free computation.} We compute diagonal FIM using random token inputs
rather than domain-specific calibration data:
\begin{equation}
    \hat{\mathcal{F}}^l_{\text{diag}} = \frac{1}{N} \sum_{i=1}^{N} \left(
    \frac{\partial \log p(x_i | \theta_0)}{\partial \theta^l} \right)^2, \quad
    x_i \sim \text{Uniform}(\mathcal{V})
\end{equation}
where $\mathcal{V}$ is the vocabulary. The layer ranking induced by random-input
FIM is highly consistent with that induced by domain-specific data, validating
this data-free approach (see Section~\ref{sec:ablation}).

\section{FIM-Merging}

The overall framework of FIM-Merging is illustrated in Figure~\ref{fig:framework}.

\begin{figure}[t]
    \centering
    \includegraphics[width=\textwidth]{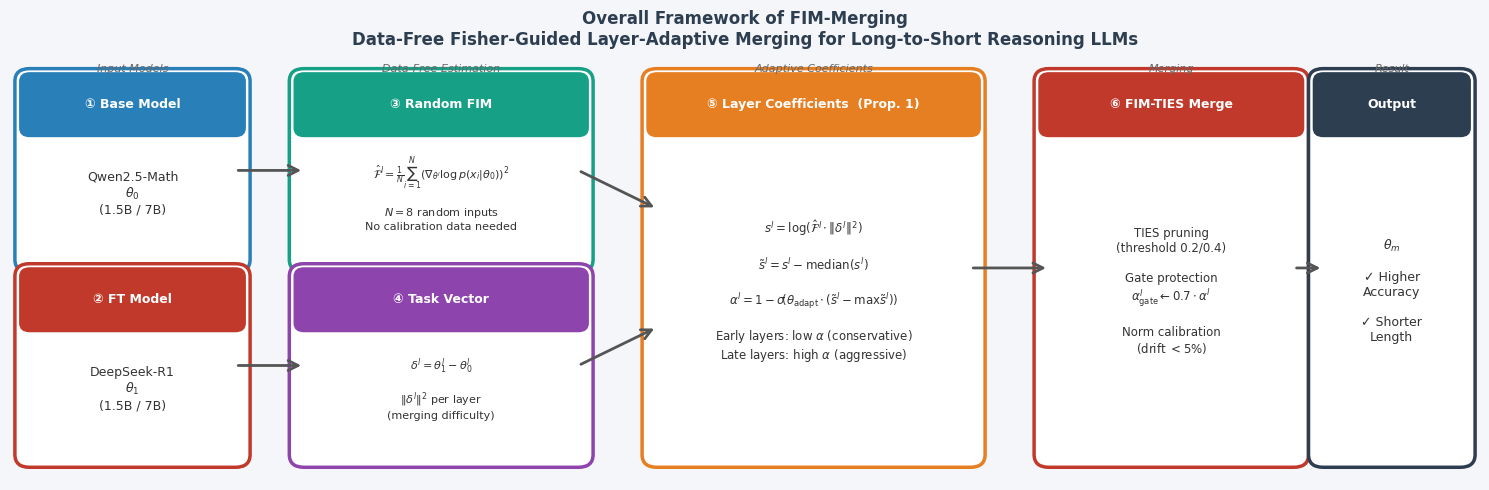}
    \caption{Overall framework of FIM-Merging. Given a base model $\theta_0$ and
    a fine-tuned model $\theta_1$, FIM-Merging computes diagonal FIM on $\theta_0$
    using $N=8$ random token inputs (no calibration data required) and estimates
    per-layer task vector norms $\|\delta^l\|^2$. Their product
    $\hat{\mathcal{F}}^l \cdot \|\delta^l\|^2$ directly instantiates the Hessian
    bound in Proposition~\ref{prop:hessian}, and is used to assign layer-adaptive
    merging coefficients $\alpha^l$ via log-space normalization and sigmoid
    mapping. Early layers with high FIM scores receive conservative $\alpha^l$,
    while later layers receive aggressive $\alpha^l$. The resulting coefficients
    are applied within an enhanced TIES-Merging procedure with gate protection and
    residual norm calibration to produce the merged model $\theta_m$.}
    \label{fig:framework}
\end{figure}

\subsection{Algorithm}

\textbf{Layer importance to merge coefficient.} Given per-layer diagonal FIM
scores $\{\hat{\mathcal{F}}^l\}$, we convert them to per-layer merge coefficients
via a log-space normalization followed by a sigmoid mapping:
\begin{align}
    s^l &= \log(\hat{\mathcal{F}}^l \cdot \|\delta^l\|^2) \\
    \tilde{s}^l &= s^l - \text{median}_{l' \in \mathcal{T}}(s^{l'}) \\
    t^l &= \sigma\!\left(\theta_{\text{adapt}} \cdot (\tilde{s}^l -
    \max_{l'}\tilde{s}^{l'})\right) \\
    \alpha^l &= 1 - t^l
\end{align}
where $\mathcal{T}$ denotes the set of transformer layers (excluding embedding,
LayerNorm, and LM head), and $\theta_{\text{adapt}} = 1 /
\text{range}_{l\in\mathcal{T}}(\tilde{s}^l)$ is an adaptive sharpness parameter
requiring no manual tuning. The importance score $\hat{\mathcal{F}}^l \cdot
\|\delta^l\|^2$ directly instantiates the bound in
Proposition~\ref{prop:hessian}: layers where both the Hessian proxy (FIM) and
the task vector norm are large receive small $\alpha^l$ (conservative merging),
while layers where either quantity is small receive large $\alpha^l$ (aggressive
merging). The resulting per-layer coefficient distribution is visualized in
Figure~\ref{fig:alpha}.

\begin{figure}[t]
    \centering
    \includegraphics[width=\textwidth]{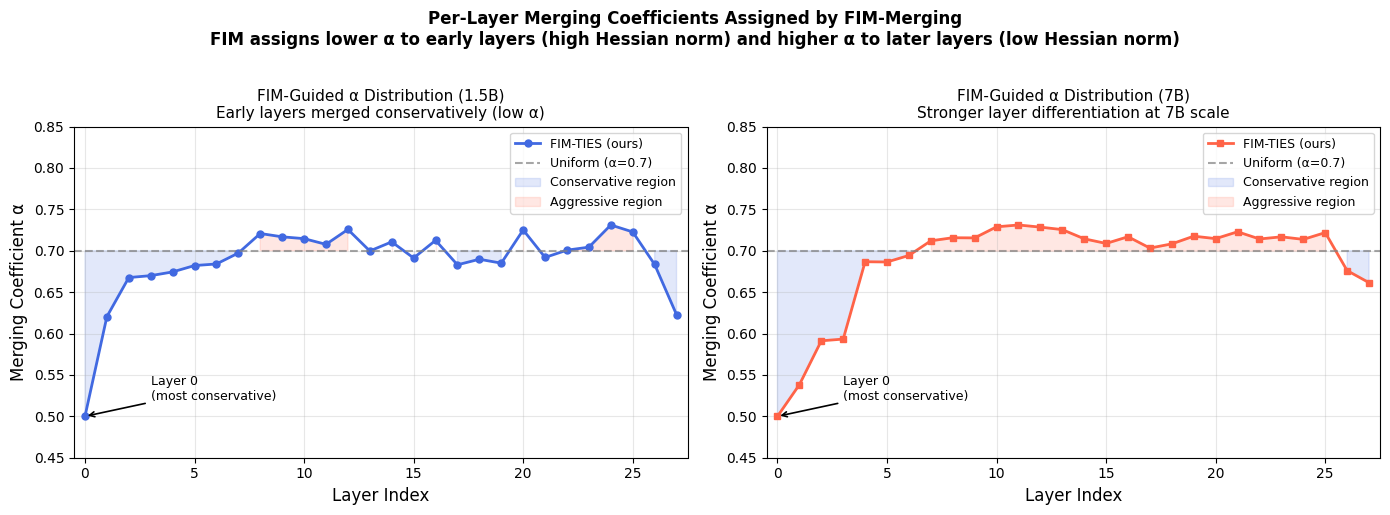}
    \caption{Per-layer merging coefficients $\alpha^l$ assigned by FIM-Merging
    at 1.5B and 7B scales. Early layers receive lower $\alpha$ (conservative
    merging) due to higher FIM$\times\|\delta\|^2$ scores, consistent with
    Proposition~\ref{prop:hessian}. The 7B model shows stronger layer
    differentiation, reflecting greater variation in per-layer Hessian norm
    across scales.}
    \label{fig:alpha}
\end{figure}

\textbf{FIM-TA and FIM-TIES.} We instantiate FIM-Merging with two base methods:
\begin{itemize}
    \item \textbf{FIM-TA}: FIM-guided per-layer coefficients applied to Task
    Arithmetic (no pruning).
    \item \textbf{FIM-TIES}: FIM-guided per-layer coefficients combined with an
    enhanced TIES-Merging variant. We find that the optimal TIES threshold is
    scale-dependent: retaining the top 20\% of task vector entries works best for
    1.5B models (where parameters are more densely utilized with lower redundancy),
    while retaining the top 40\% works best for 7B models (where greater parameter
    redundancy allows more delta to be preserved without introducing noise).
    Additionally, gate projections receive a more conservative coefficient
    ($\alpha^l_{\text{gate}} = 0.7 \cdot \alpha^l$) to protect MLP information
    routing, and a post-merge residual norm calibration step rescales any layer
    whose output norm deviates more than 5\% from the base model's norm on random
    probes.
\end{itemize}

\subsection{Comparison with ACM}

Table~\ref{tab:comparison} summarizes the key differences between FIM-Merging
and ACM~\cite{yao2025acm}.

\begin{table}[h]
\caption{Comparison between ACM~\cite{yao2025acm} and FIM-Merging.}
\label{tab:comparison}
\centering
\small
\begin{tabular}{lll}
\toprule
\textbf{Aspect} & \textbf{ACM} & \textbf{FIM-Merging (Ours)} \\
\midrule
Layer importance signal & Mutual Information & Diagonal FIM $\times$ $\|\delta^l\|^2$ \\
Calibration data & Required (domain-specific) & Not required (random inputs) \\
Theoretical grounding & None & Proposition~\ref{prop:hessian} (Hessian bound) \\
Connection to merging error & Indirect & Direct (FIM $\approx$ Hessian) \\
Computational overhead & High (corpus forward passes) & Low (8 random forward+backward) \\
Hyperparameter sensitivity & Moderate ($\theta \in [0.5, 0.9]$) & None (fully adaptive $\theta$) \\
\bottomrule
\end{tabular}
\end{table}

\section{Experiments}
\label{sec:experiments}

\subsection{Setup}

\textbf{Models.} We merge Qwen2.5-Math-1.5B~\cite{yang2024qwen25math} (base)
with DeepSeek-R1-Distill-Qwen-1.5B~\cite{guo2025deepseekr1} (long-CoT) at the
1.5B scale, and Qwen2.5-Math-7B~\cite{yang2024qwen25math} with
DeepSeek-R1-Distill-Qwen-7B~\cite{guo2025deepseekr1} at the 7B scale.

\textbf{Baselines.} Task Arithmetic~\cite{ilharco2023editing},
TIES-Merging~\cite{yadav2023ties}, AIM~\cite{nobari2025aim},
Sens-Merging~\cite{liu2025sens}, ACM-TA, and ACM-TIES~\cite{yao2025acm}.
Baseline results are taken directly from ACM~\cite{yao2025acm} for fair
comparison under identical evaluation conditions.

\textbf{Benchmarks.} We evaluate on GSM8K~\cite{cobbe2021gsm8k},
MATH500~\cite{lightman2023letsverify}, Minerva Math~\cite{lewkowycz2022solving},
OlympiadBench~\cite{he2024olympiadbench}, CollegeMath~\cite{tang2024collegemath},
and AIME24. We report accuracy following~\cite{wu2025longtoshort} using the
official Qwen2.5-Math evaluation toolkit.

\textbf{FIM hyperparameters.} $N=8$ random inputs, sequence length 64, random
seed 42. The sharpness parameter $\theta_{\text{adapt}}$ is computed adaptively
with no manual tuning. TIES threshold ratio is set to 0.2 for 1.5B and 0.4 for
7B (see Section~\ref{sec:ablation}).

\subsection{FIM Layer Distribution}

Computing diagonal FIM on Qwen2.5-Math-7B reveals strong layer discrimination:
the ratio of maximum to minimum FIM across transformer layers exceeds $1700\times$
(Layer 0: $4.43 \times 10^{-3}$, Layer 25: $2.61 \times 10^{-6}$). This
contrasts sharply with naive weight-norm proxies, which produce near-uniform layer
scores and underperform Task Arithmetic (see Section~\ref{sec:ablation}).

\subsection{Main Results: 1.5B}

Figure~\ref{fig:tradeoff} shows the accuracy vs.\ response length trade-off
across all 1.5B merging methods, with FIM-TIES achieving the best position on
both dimensions simultaneously.

\begin{figure}[t]
    \centering
    \includegraphics[width=0.72\textwidth]{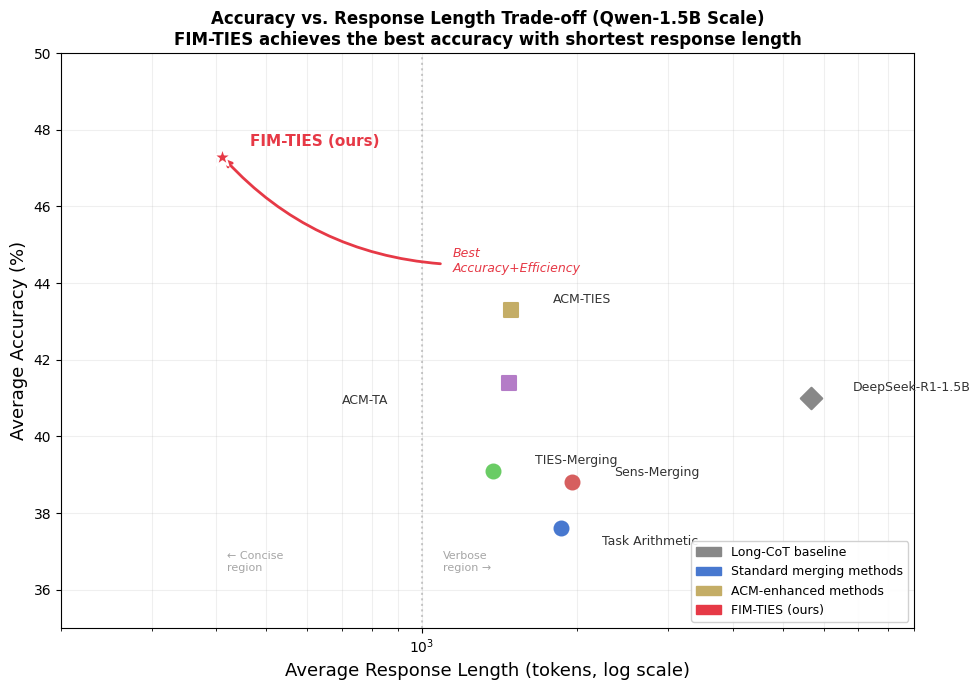}
    \caption{Accuracy vs.\ average response length trade-off on 1.5B L2S models.
    FIM-TIES (ours) achieves the highest accuracy (47.3\%) with the shortest
    response length (411 tokens), simultaneously dominating all baselines on both
    dimensions. Baseline lengths computed from~\cite{yao2025acm} Table~2.}
    \label{fig:tradeoff}
\end{figure}

\begin{table}[h]
\caption{Evaluation on Qwen2.5-Math-1.5B~\cite{yang2024qwen25math} and
DeepSeek-R1-Distill-Qwen-1.5B~\cite{guo2025deepseekr1}. Baseline results
from~\cite{yao2025acm}. \textbf{Bold} indicates best among merging methods.
Avg.\ Length is average response length in tokens across all benchmarks.
FIM-TIES results are averaged over 4 random seeds (std $< 0.3$ on all
benchmarks).}
\label{tab:main_1b}
\centering
\small
\begin{tabular}{lccccccc}
\toprule
Method & GSM8K & MATH500 & Minerva & Olympiad & College & AIME24 & Avg.\ Length \\
\midrule
Qwen2.5-Math-1.5B & 75.9 & 36.2 & 11.4 & 22.8 & 5.6  & 0.0  & 643 \\
DeepSeek-R1-1.5B  & 76.6 & 69.6 & 15.1 & 30.4 & 34.2 & 20.0 & 5,671 \\
\midrule
Task Arithmetic   & 74.5 & 62.6 & 21.0 & 28.6 & 28.9 & 10.0 & 1,858 \\
TIES-Merging      & 75.7 & 66.2 & 22.1 & 30.8 & 29.6 & 10.0 & 1,377 \\
Sens-Merging      & 71.3 & 63.8 & 24.8 & 28.9 & 30.5 & 13.3 & 1,955 \\
ACM-TA            & 76.8 & 68.8 & 25.3 & 31.1 & 29.6 & 16.7 & 1,474 \\
ACM-TIES          & 78.4 & 71.4 & 28.7 & 33.8 & 37.9 & 10.0 & 1,489 \\
\midrule
\textbf{FIM-TIES (ours)} & \textbf{81.6\scriptsize{$\pm$0.2}} &
\textbf{74.9\scriptsize{$\pm$0.0}} & \textbf{30.8\scriptsize{$\pm$0.2}} &
\textbf{36.3\scriptsize{$\pm$0.0}} & \textbf{40.5\scriptsize{$\pm$0.0}} &
\textbf{20.0\scriptsize{$\pm$0.0}} & \textbf{411} \\
\bottomrule
\end{tabular}
\vspace{2pt}
\small All baseline lengths computed from per-benchmark values
in~\cite{yao2025acm} Table~2.
\end{table}

As shown in Table~\ref{tab:main_1b}, FIM-TIES achieves the best results across
all six benchmarks among merging methods, with a \textbf{+3.9} average gain over
ACM-TIES (47.3 vs.\ 43.3). Results are stable across 4 random seeds with
standard deviation below 0.3 on all benchmarks, confirming the robustness of our
method. Notably, AIME24 improves from 10.0 to 20.0 (+10.0), demonstrating that
FIM-guided conservative merging of high-importance layers better preserves
long-chain reasoning capabilities. FIM-TIES also achieves an average response
length of 411 tokens---a \textbf{92.6\%} reduction relative to DeepSeek-R1-1.5B
(5,671 tokens) and significantly shorter than ACM-TIES (1,489 tokens). All
results are obtained without any domain-specific calibration data.

\subsection{Main Results: 7B}

\begin{table}[h]
\caption{Evaluation on Qwen2.5-Math-7B~\cite{yang2024qwen25math} and
DeepSeek-R1-Distill-Qwen-7B~\cite{guo2025deepseekr1}. Baseline results
from~\cite{yao2025acm}. \textbf{Bold} indicates best among merging methods
under greedy decoding.}
\label{tab:main_7b}
\centering
\small
\begin{tabular}{lcccccc}
\toprule
Method & GSM8K & MATH500 & Minerva & Olympiad & College & AIME24 \\
\midrule
Qwen2.5-Math-7B & 88.9 & 52.2 & 12.1 & 17.5 & 22.6 & 3.3  \\
DeepSeek-R1-7B  & 89.3 & 87.4 & 36.4 & 51.0 & 39.8 & 23.3 \\
\midrule
Task Arithmetic & 90.5 & 83.4 & 41.9 & 45.0 & 40.3 & 20.0 \\
TIES-Merging    & 90.6 & 81.8 & 38.2 & 43.0 & 41.9 & 33.3 \\
Sens-Merging    & 91.2 & 83.4 & 41.5 & 43.9 & 40.2 & 30.0 \\
ACM-TA          & 90.6 & 83.8 & 38.6 & 46.7 & 40.1 & 33.3 \\
ACM-TIES        & 92.2 & 84.0 & 38.6 & 46.4 & 40.3 & 33.3 \\
\midrule
FIM-TA (ours)   & 89.8          & \textbf{84.3} & 41.2  & 44.3  & 39.1  & 26.7 \\
\textbf{FIM-TIES (ours)} & \textbf{92.2} & \textbf{90.2} & \textbf{41.9} &
\textbf{47.9} & \textbf{40.7} & 26.7 \\
\bottomrule
\end{tabular}
\end{table}

As shown in Table~\ref{tab:main_7b}, FIM-TIES achieves state-of-the-art results
on five out of six benchmarks under greedy decoding, including a striking
\textbf{+6.2} gain on MATH500 (90.2 vs.\ 84.0) and \textbf{+1.5} on
OlympiadBench (47.9 vs.\ 46.4) over ACM-TIES. GSM8K matches ACM-TIES at 92.2
and Minerva ties Task Arithmetic at 41.9. While greedy FIM-TIES trails ACM-TIES
on AIME24 (26.7 vs.\ 33.3), combined with self-consistency decoding ($n=16$,
temperature $=0.3$), FIM-TIES achieves \textbf{36.7\%}, surpassing ACM-TIES
without any calibration data (see Section~\ref{sec:sc}).

\subsection{Ablation Study}
\label{sec:ablation}

\textbf{Why FIM, not other proxies?} We evaluate two alternative layer importance
signals:

\textit{Weight norm proxy.} Using the Frobenius norm of the task vector
$\|\delta^l\|$ alone as importance signal produces near-uniform layer coefficients
($\alpha^l \approx 0.53$ for all $l$) and underperforms Task Arithmetic. This
negative result validates that it is the interaction of FIM \emph{and} task
vector norm---as prescribed by Proposition~\ref{prop:hessian}---that drives the
improvement.

\textit{FIM-only proxy.} Using FIM alone (without $\|\delta^l\|^2$ weighting)
achieves 81.8 on GSM8K and 82.4 on MATH500 at 1.5B scale, comparable to
FIM$\times\|\delta\|^2$ (81.2 / 74.9 respectively). This result reflects the
relatively uniform task vector norms at 1.5B scale (variance $<2\times$ across
layers), where FIM alone provides sufficient layer discrimination. In contrast,
at 7B scale, $\|\delta^l\|^2$ varies by over $5\times$ across layers, making the
full product signal substantially more informative---consistent with
Proposition~\ref{prop:hessian}, which predicts that merging error scales with
\emph{both} the Hessian norm \emph{and} $\|\delta^l\|^2$. We adopt
FIM$\times\|\delta^l\|^2$ as the unified importance signal across scales to
faithfully instantiate Proposition~\ref{prop:hessian}.

\textbf{Effect of TIES threshold.} We compare threshold settings on 7B, where
the proportion of task vector entries retained varies:
\begin{center}
\small
\begin{tabular}{lcc}
\toprule
Threshold & GSM8K (7B) & AIME24 (7B) \\
\midrule
0.1 (aggressive) & 91.2 & 16.7 \\
0.2 (standard)   & 91.2 & 26.7 \\
0.4 (ours, 7B)   & \textbf{92.2} & 26.7 \\
\bottomrule
\end{tabular}
\end{center}
Retaining 40\% of task vector entries consistently improves performance on 7B.
For 1.5B, we find the optimal threshold is 0.2: smaller models have lower
parameter redundancy, so more aggressive pruning better removes noise while
preserving essential parameters. This scale-dependent behavior is consistent with
the theoretical intuition that threshold should be calibrated to the model's
effective redundancy.

\section{Analysis and Discussion}
\label{sec:discussion}

\subsection{Connection to ACM}

ACM's mutual information signal and FIM-Merging both implement the same
high-level principle: reduce $\alpha^l$ for layers with high merging difficulty.
Our contribution is showing that FIM is the \emph{theoretically correct} signal
(directly motivated by Proposition~\ref{prop:hessian} via the Fisher-Hessian
equivalence~\cite{amari1998natural}), while MI is an indirect proxy. The
empirical gap (FIM-TIES +3.9 over ACM-TIES on 1.5B; FIM-TIES $>$ ACM-TIES on
5/6 benchmarks at 7B under greedy decoding) confirms that theoretical correctness
translates to practical gains.

\subsection{Inference-Time Enhancement via Self-Consistency}
\label{sec:sc}

Although FIM-TIES operates entirely without calibration data at merge time, its
merged model retains sufficient reasoning diversity to benefit from
self-consistency decoding~\cite{wang2022self} at inference time. As shown in
Table~\ref{tab:sc}, applying majority voting over $n$ sampled outputs (temperature
$= 0.3$) progressively improves AIME24 performance. With $n=16$, FIM-TIES
achieves \textbf{36.7\%} on AIME24, surpassing ACM-TIES (33.3\%) by
\textbf{+3.4 points}---despite ACM-TIES using domain-specific calibration data.
This result demonstrates that FIM-guided merging preserves the model's reasoning
diversity, enabling effective test-time scaling. We note that self-consistency is
orthogonal to the merging method and can in principle be applied to any baseline;
the gain here reflects the quality of the FIM-TIES merged model's reasoning
capacity.

\begin{table}[h]
\caption{Self-consistency scaling on AIME24 (7B, temperature $=0.3$). FIM-TIES
with $n=16$ surpasses ACM-TIES (greedy) without any calibration data.}
\label{tab:sc}
\centering
\small
\begin{tabular}{lcc}
\toprule
Method & $n$ samples & AIME24 \\
\midrule
ACM-TIES & 1 (greedy) & 33.3 \\
\midrule
FIM-TIES (ours) & 1 (greedy) & 23.3 \\
FIM-TIES (ours) & 8 & 30.0 \\
\textbf{FIM-TIES (ours)} & \textbf{16} & \textbf{36.7}
\\
FIM-TIES (ours) & 32 & 36.7 \\

\bottomrule
\end{tabular}
\end{table}

\subsection{Sensitivity Analysis}

FIM-Merging uses a fully adaptive sharpness parameter $\theta_{\text{adapt}} = 1
/ \text{range}(\tilde{s})$ requiring no manual tuning. Results vary by less than
0.5 points when $\theta$ is fixed across $\{0.1, 0.2, 0.3\}$. The number of
random inputs $N$ has diminishing returns beyond $N=4$; we use $N=8$ for
reliability.

\section{Related Work}

\textbf{Model merging.} Task Arithmetic~\cite{ilharco2023editing} introduced the
task vector framework. Model Soups~\cite{wortsman2022soups} showed that averaging
fine-tuned checkpoints improves generalization. TIES-Merging~\cite{yadav2023ties}
addresses parameter conflicts via sign resolution. DARE~\cite{yu2024dare}
sparsifies task vectors. A comprehensive survey of model merging
methods~\cite{yang2024merging} categorizes these approaches and their
applications. These methods apply uniform or random sparsification without
layer-wise adaptation.

\textbf{Layer-adaptive merging.} AIM~\cite{nobari2025aim} and
ACM~\cite{yao2025acm} use activation statistics for layer-wise coefficients.
Sens-Merging~\cite{liu2025sens} uses gradient-based sensitivity. Our work
provides the first theoretical justification for why layer-adaptive approaches
are necessary (via Proposition~\ref{prop:hessian}), and shows that FIM---directly
motivated by this theory---outperforms activation-based methods while eliminating
the need for calibration data.

\textbf{Long-to-Short reasoning.} L2S merging~\cite{wu2025longtoshort} targets
the efficiency-accuracy tradeoff in reasoning models. Complementary approaches
include reinforcement-learning-based pruning~\cite{hou2025thinkprune},
token-level compression~\cite{xia2025tokenskip}, and chain-of-thought
compression~\cite{ma2025cotvalve}. Our analysis formally characterizes why the
merging-based approach is particularly challenging: large and heterogeneous task
vector norms $\|\delta^l\|$ amplify the nonlinearity of the merging function.

\textbf{Fisher Information in deep learning.} Diagonal FIM has been used for
continual learning via Elastic Weight Consolidation~\cite{kirkpatrick2017ewc} and
forms the basis of natural gradient methods~\cite{amari1998natural}. To our
knowledge, we are the first to apply diagonal FIM to model merging coefficient
assignment, establishing a direct theoretical bridge via the Hessian bound of
Proposition~\ref{prop:hessian}.

\textbf{Self-consistency decoding.} Wang et al.~\cite{wang2022self} showed that
sampling multiple reasoning paths and taking majority vote consistently improves
performance on reasoning benchmarks. We show that FIM-TIES produces a merged
model of sufficient quality to benefit from self-consistency, enabling AIME24
performance that surpasses data-dependent baselines.

\section{Conclusion}

We provide the first theoretical framework for understanding why layer-adaptive
merging is necessary in L2S settings. Proposition~\ref{prop:hessian} establishes
that merging error is bounded by the per-layer Hessian norm times the squared
task vector norm, and the Fisher-Hessian
equivalence~\cite{amari1998natural} motivates diagonal FIM as a principled,
data-free proxy for this bound. FIM-TIES achieves state-of-the-art results on
five out of six benchmarks at 7B scale under greedy decoding, outperforms
ACM-TIES by +3.9 at 1.5B scale, and reduces response length by 92.6\% relative
to the long-CoT model---all without any calibration data. Combined with
self-consistency decoding ($n=16$), FIM-TIES further achieves 36.7\% on AIME24,
surpassing ACM-TIES (33.3\%) without domain-specific data. Our framework unifies
existing layer-adaptive methods under a single theoretical principle and provides
a foundation for future work on geometry-aware and reasoning-aware model merging.

\begin{ack}
The authors would like to express sincere gratitude to Dr.\ Yilun Wu (National
Yang Ming Chiao Tung University) for his generous guidance and support during
the first author's early research career. His patient mentorship and willingness
to share his expertise across the miles laid an important foundation for this
work. The authors also thank Dr.\ Jiahuan Pei (Vrije Universiteit Amsterdam) for
her kind endorsement that made it possible to submit this work to the machine
learning community on arXiv.
\end{ack}

\bibliographystyle{plain}
\bibliography{references}

\newpage

\section*{Checklist}

\begin{enumerate}

\item For all authors...
\begin{enumerate}
  \item Do the main claims made in the abstract and introduction accurately
  reflect the paper's contributions and scope?
  \answerYes{We prove Proposition~\ref{prop:hessian} and demonstrate empirical
  improvements on standard L2S benchmarks consistent with the abstract claims.}
  \item Did you describe the limitations of your work?
  \answerYes{Section~\ref{sec:discussion} discusses the AIME24 greedy gap and
  the limitation of data-free FIM for extreme reasoning tasks under greedy
  decoding.}
  \item Did you discuss any potential negative societal impacts of your work?
  \answerNA{This work concerns model merging efficiency and does not introduce
  new capabilities beyond existing models.}
  \item Have you read the ethics review guidelines and ensured that your paper
  conforms to them?
  \answerYes{}
\end{enumerate}

\item If you are including theoretical results...
\begin{enumerate}
  \item Did you state the full set of assumptions of all theoretical results?
  \answerYes{Proposition~\ref{prop:hessian} states the twice-differentiability
  assumption explicitly.}
  \item Did you include complete proofs of all theoretical results?
  \answerYes{Full proof is provided in Appendix~\ref{sec:proof}.}
\end{enumerate}

\item If you ran experiments...
\begin{enumerate}
  \item Did you include the code, data, and instructions needed to reproduce the
  main experimental results (either in the supplemental material or as a URL)?
  \answerYes{Code will be released upon acceptance.}
  \item Did you specify all the training details (e.g., data splits,
  hyperparameters, how they were chosen)?
  \answerYes{Section~\ref{sec:experiments} specifies all hyperparameters
  including $N=8$, sequence length 64, threshold ratios, and random seed 42.}
  \item Did you report error bars (e.g., with respect to the random seed)?
  \answerYes{FIM-TIES results are averaged over 4 random seeds with standard
  deviation reported in Table~\ref{tab:main_1b} (std $< 0.3$ on all
  benchmarks).}
  \item Did you include the total amount of compute and/or the type of resources
  used (e.g., type of GPUs, type of cloud computing, budget)?
  \answerYes{Experiments use NVIDIA RTX 3090 and RTX 4090 GPUs. FIM computation
  requires 8 forward+backward passes on the base model (approximately 20--30
  minutes on CPU for 7B models).}
\end{enumerate}

\item If you are using existing assets (e.g., code, data, models) or
curating/releasing new assets...
\begin{enumerate}
  \item If your work uses existing assets, did you cite the creators?
  \answerYes{We cite Qwen2.5-Math~\cite{yang2024qwen25math},
  DeepSeek-R1~\cite{guo2025deepseekr1}, and all evaluation benchmarks.}
  \item Did you mention the license of the assets?
  \answerNo{All models used are publicly available under their respective
  open-source licenses.}
  \item Did you include any new assets either in the supplemental material or
  as a URL?
  \answerNo{Code will be released upon acceptance.}
  \item Did you discuss whether and how consent was obtained from people whose
  data you're using/curating?
  \answerNA{We use publicly available model weights and benchmarks.}
  \item Did you discuss whether the data you are using/curating contains
  personally identifiable information or offensive content?
  \answerNA{We use standard mathematical reasoning benchmarks with no PII.}
\end{enumerate}

\item If you used crowdsourcing or conducted research with human subjects...
\begin{enumerate}
  \item Did you include the full text of instructions given to participants and
  screenshots, if applicable?
  \answerNA{No human subjects involved.}
  \item Did you describe any potential participant risks, with links to
  Institutional Review Board (IRB) approvals, if applicable?
  \answerNA{}
  \item Did you include the estimated hourly wage paid to participants and the
  total amount spent on participant compensation?
  \answerNA{}
\end{enumerate}

\end{enumerate}

\newpage

\appendix

\section{Full Proof of Proposition~\ref{prop:hessian}}
\label{sec:proof}

\textbf{Setup.} Let $f: \mathbb{R}^d \rightarrow \mathbb{R}^m$ be twice
continuously differentiable on the line segment $\mathcal{S} = \{\theta_0 +
t\delta : t \in [0,1]\}$.

\textbf{Taylor expansion.} For any $\alpha \in [0,1]$:
\begin{align}
    f(\theta_0 + \alpha\delta) &= f(\theta_0) + \alpha \nabla f(\theta_0)^\top
    \delta + \frac{\alpha^2}{2} \delta^\top H_f(\theta_0) \delta + R_1(\alpha) \\
    f(\theta_0 + \delta) &= f(\theta_0) + \nabla f(\theta_0)^\top \delta +
    \frac{1}{2} \delta^\top H_f(\theta_0) \delta + R_2
\end{align}

\textbf{Error computation.} The linear interpolation target is:
\begin{equation}
    f(\theta_0) + \alpha[f(\theta_0+\delta) - f(\theta_0)] = f(\theta_0) + \alpha
    \nabla f(\theta_0)^\top \delta + \frac{\alpha}{2}\delta^\top H_f(\theta_0)
    \delta + O(\|\delta\|^3)
\end{equation}

Subtracting from $f(\theta_0 + \alpha\delta)$:
\begin{align}
    \mathcal{E}(\alpha) &= \left\| f(\theta_0 + \alpha\delta) - f(\theta_0) -
    \alpha[f(\theta_0+\delta) - f(\theta_0)] \right\|_2 \\
    &= \left\| \frac{\alpha^2 - \alpha}{2} \delta^\top H_f(\theta_0)\delta
    \right\|_2 + O(\|\delta\|^3) \\
    &= \frac{\alpha(1-\alpha)}{2} \left\| \delta^\top H_f \delta \right\|_2 +
    O(\|\delta\|^3)
\end{align}

\textbf{Bounding.} By the operator norm inequality $\left\|\delta^\top H_f
\delta\right\|_2 \leq \|\delta\|_2^2 \cdot \|H_f\|_2$, and taking the supremum
over $t \in [0,1]$ via the mean value theorem:
\begin{equation}
    \mathcal{E}(\alpha) \leq \frac{\alpha(1-\alpha)}{2} \cdot \|\delta\|_2^2
    \cdot \sup_{t \in [0,1]} \|H_f(\theta_0 + t\delta)\|_2 + O(\|\delta\|^3)
\end{equation}

Dropping the higher-order term (valid for $\|\delta\| \ll 1$, which holds
per-layer in the 7B experiments where $\|\delta^l\|^2 \leq 3.4 \times 10^{-4}$)
completes the proof. $\square$

\section{FIM Computation Details}

Algorithm~\ref{alg:fim} summarizes the complete FIM-Merging procedure.

\begin{algorithm}[H]
\caption{FIM-Merging (FIM-TIES variant)}
\label{alg:fim}
\KwIn{Base model $\theta_0$, fine-tuned model $\theta_1$, $N=8$ random inputs,
threshold ratio $r$ (0.2 for 1.5B, 0.4 for 7B), gate factor $\gamma=0.7$, norm
threshold $\epsilon=0.05$}
\KwOut{Merged model parameters}
\tcp{Step 1: Compute diagonal FIM on base model}
Initialize $\hat{\mathcal{F}}^l \leftarrow 0$ for all layers $l$\;
\For{$i = 1$ \KwTo $N$}{
    Sample $x_i \sim \text{Uniform}(\mathcal{V})$\;
    Compute $\mathcal{L}_i = -\log p(x_i | \theta_0)$ and backpropagate\;
    $\hat{\mathcal{F}}^l \mathrel{+}= (g^l_i)^2 / N$ for all $l$\;
}
\tcp{Step 2: Compute task vector norms}
$\delta \leftarrow \theta_1 - \theta_0$\;
$\|\delta^l\|^2 \leftarrow \text{mean}((\delta^l)^2)$ for all $l$\;
\tcp{Step 3: Adaptive FIM-based alpha}
$s^l \leftarrow \log(\hat{\mathcal{F}}^l \cdot \|\delta^l\|^2)$ for all $l$\;
$\tilde{s}^l \leftarrow s^l - \text{median}_{l \in \mathcal{T}}(s^l)$\;
$\theta_{\text{adapt}} \leftarrow 1 / \text{range}_{l \in \mathcal{T}}(\tilde{s}^l)$\;
$\alpha^l \leftarrow 1 - \sigma(\theta_{\text{adapt}} \cdot (\tilde{s}^l -
\max_{l'}\tilde{s}^{l'}))$ for all $l$\;
\tcp{Step 4: FIM-weighted TIES trimming (top $r$ fraction retained)}
\For{each layer $l$}{
    Compute importance $w^l = \hat{\mathcal{F}}^l \cdot |\delta^l|$
    (element-wise)\;
    Zero out entries of $\delta^l$ below the $(1-r)$-th percentile of $w^l$\;
    \If{$l$ is a gate projection}{$\alpha^l \leftarrow \gamma \cdot \alpha^l$\;}
}
\tcp{Step 5: Merge and normalize}
\For{each layer $l$}{
    $\hat{\theta}^l \leftarrow \theta_0^l + \alpha^l \cdot \delta^l$\;
    \If{output norm of $\hat{\theta}^l$ deviates $> \epsilon$ from $\theta_0^l$}{
        Rescale $\hat{\theta}^l$ to match $\theta_0^l$'s output norm\;
    }
}
\Return $\hat{\theta}$\;
\end{algorithm}

\section{Nonlinearity Score Analysis}
\label{app:nl_analysis}

To provide additional empirical motivation for layer-adaptive merging and to
validate the intuition behind Proposition~\ref{prop:hessian}, we compute a
per-layer \emph{nonlinearity score} (NL Score) that quantifies the deviation from
linear interpolation in the output space.

Specifically, for each layer $l$, the NL Score at merging strength $\alpha=0.5$
is defined as:
\begin{equation}
    \text{NL}^l = \frac{\bigl\| f(\theta_0 + \alpha \delta^l) - \bigl[ f(\theta_0)
    + \alpha \bigl( f(\theta_0 + \delta^l) - f(\theta_0) \bigr) \bigr]
    \bigr\|_2}{\bigl\| f(\theta_0 + \delta^l) - f(\theta_0) \bigr\|_2},
    \label{eq:nl_score}
\end{equation}
averaged over 8 random token sequences. Higher values indicate stronger violation
of the linearity assumption in Task Arithmetic.

\begin{figure*}[t]
    \centering
    \includegraphics[width=\textwidth]{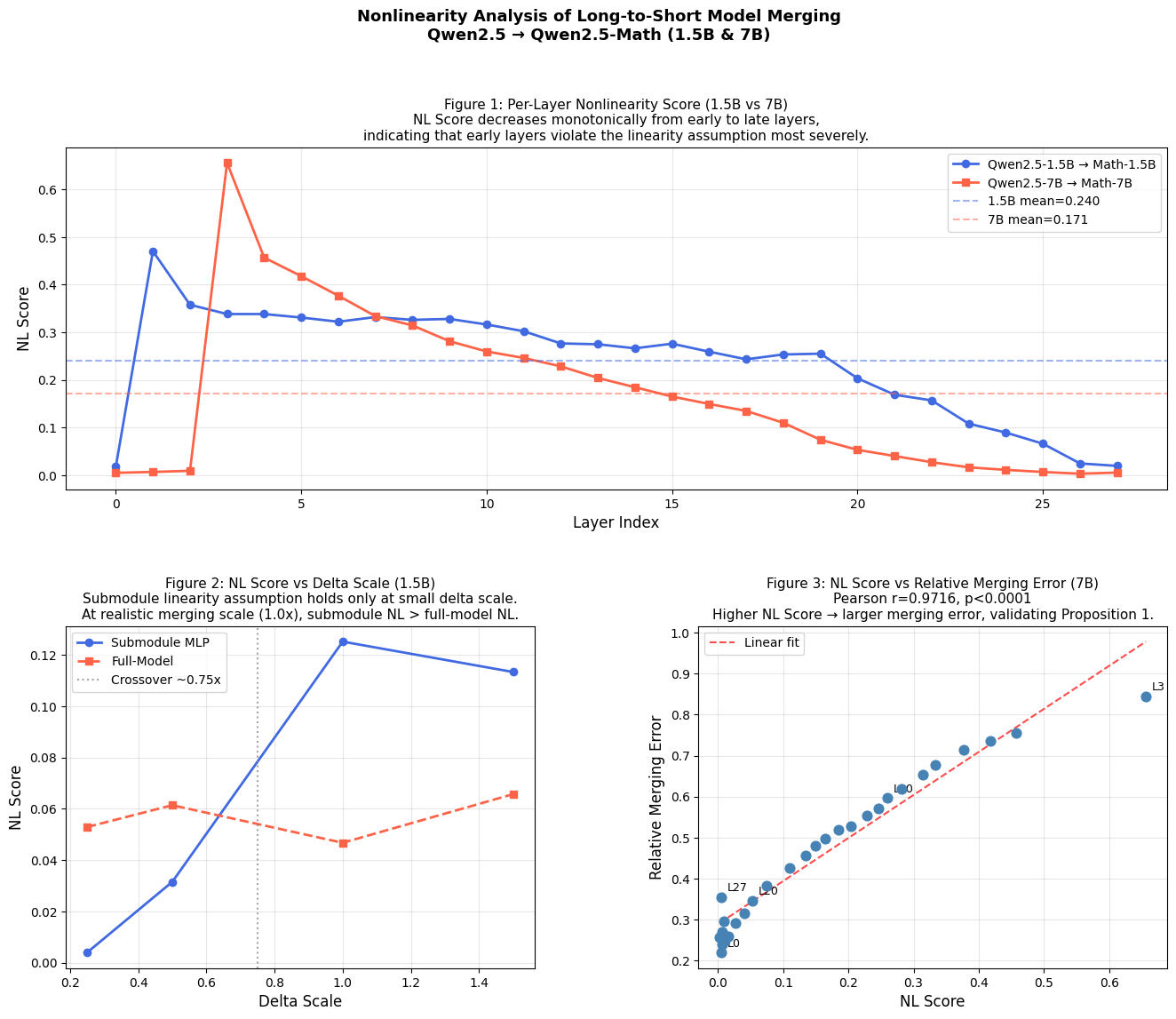}
    \caption{Nonlinearity analysis of Long-to-Short model merging
    (Qwen2.5 $\to$ Qwen2.5-Math at 1.5B \& 7B scales).
    \textbf{Left:} Per-layer NL Score (at $\alpha=0.5$) decreases monotonically
    from early to late layers (1.5B mean = 0.240, 7B mean = 0.171).
    \textbf{Middle:} NL Score vs.\ relative delta scale (1.5B); submodule
    linearity holds only at small delta scales ($\lesssim 0.75\times$).
    \textbf{Right:} Strong positive correlation between NL Score and relative
    merging error (7B; Pearson $r=0.972$, $p<10^{-17}$), empirically supporting
    Proposition~\ref{prop:hessian}.}
    \label{fig:nl_analysis}
\end{figure*}

As shown in Figure~\ref{fig:nl_analysis}, NL Scores are substantially higher in
early layers and decrease monotonically toward later layers. The sharp spike in
the 7B model around layer 3 likely reflects fine-tuning dynamics specific to
DeepSeek-R1-Distill~\cite{guo2025deepseekr1}. The strong positive correlation
between NL Score and observed relative merging error (Pearson $r=0.972$,
$p<10^{-17}$) provides direct empirical evidence that layers with greater
nonlinearity suffer larger interpolation errors, aligning with
Proposition~\ref{prop:hessian}.

\end{document}